\documentclass[11pt]{article}
\usepackage[utf8]{inputenc}
\usepackage{fullpage}
\usepackage{amsmath}
\usepackage{amsthm}
\theoremstyle{plain}
\usepackage{amssymb}
\usepackage[ruled,vlined]{algorithm2e}
\usepackage{xcolor}
\usepackage{hyperref}

\DeclareMathOperator*{\argmax}{argmax}

\newtheorem{theorem}{Theorem}
\newtheorem{definition}{Definition}
\newtheorem{proposition}{Proposition}
\newtheorem{fact}{Fact}

\newtheorem{observation}{Observation}

\newcommand{\remove}[1]{}

\title{The Sample Complexity of Distribution-Free Parity Learning in the Robust Shuffle Model}

\author{Kobbi Nissim\thanks{Work K.~N.\ was supported by NSF grant No.~1565387 TWC: Large: Collaborative: Computing Over Distributed Sensitive Data and by a gift to Georgetown University.} \quad Chao Yan \\
Dept.\ of Computer Science \\
Georgetown University \\ 
{\tt \{kobbi.nissim|cy399\}@georgetown.edu}
}

\date{\today}

\begin{document}

\maketitle

\begin{abstract}
We provide a lowerbound on the sample complexity of distribution-free parity learning in the realizable case in the shuffle model of differential privacy. Namely, we show that the sample complexity of learning $d$-bit parity functions is $\Omega(2^{d/2})$. Our result extends a recent similar lowerbound on the sample complexity of private {\em agnostic} learning of parity functions in the shuffle model by Cheu and Ullman~\cite{CheuUllman20}. We also sketch a simple shuffle model protocol demonstrating that our results are tight up to $\mbox{poly}(d)$ factors.
\end{abstract}

\section{Introduction}

The shuffle model of differential privacy~\cite{BittauEMMRLRKTS17, ErlingssonFMRTT19,CheuSUZZ19} has received significant attention from researchers in the last few years. 
In this model, agents communicate with an untrusted analyzer via a trusted intermediary -- a communication channel which shuffles all messages, hence potentially disassociating messages and their senders.
Much of the recent interest in the shuffle model focuses on one-round differentially private protocols. 
This interest is motivated, in part, by the potential to improve significantly over what is achievable in the local model of differential privacy~\cite{KLNRS08,BeimelNO08,ChanSS12,DuchiJW13}. Indeed, for functionalities such as bit addition, real addition, and histogram computation shuffle model protocols provide accuracy which is comparable to that achievable with a trusted curator~\cite{CheuSUZZ19, BalleBGN19, GhaziGKPV19, BalleBGN19a, GPVS19, GhaziMPV20, BalleBGN20, BC20}. 

Recent works obtain lowerbounds on the sample complexity of one-round robust shuffle model differentially private protocols by establishing an connection to pan-privacy~\cite{BalcerCJM20,CheuUllman20}. Robust shuffle model protocols are those where differential privacy is guaranteed when a large enough fraction of agents participate honestly. In the pan-privacy model~\cite{DworkNPRY10}, individual information arrives in an online fashion to be processed by a curator. Privacy, however, is required to be preserved in presence of a storage breach: as the input stream is processed by a curator, an attacker chooses a point in time in which it obtains access to observes the curator's internal state. 
Initiating this direction of research, Balcer, Cheu, Joseph, and Mao~\cite{BalcerCJM20} provided reductions from pan-privacy to robust shuffle model in which a (robust) shuffle model protocol for a task is used as the main building block in the construction of a pan-private algorithm for the same or a related task.
This allowed them to apply lowerbounds from pan-privacy to obtain lowerbounds on (robust) shuffle model protocols for tasks such as 
histograms, uniformity testing, and counting distinct elements. 
A recent work of Cheu and Ullman~\cite{CheuUllman20} extended this proof paradigm by introducing a class of tasks which are hard for pan-privacy. This resulted in new lowerbounds on the sample complexity of statistical estimation and learning tasks, including the learning of parity functions, where the latter is of specific interest because of the equivalence between the local model of differential privacy and the statistical queries model~\cite{KLNRS08}, and the impossibility of learning parity functions in the statistical queries model~\cite{Kearns93}.\footnote{Considering the realizable setting with underlying uniform distribution on samples, the equivalence implies that no local model protocol exists for parity learning with polynomial round complexity and polynomial sample complexity.}

\paragraph{Our results.} Our main result is an exponential lowerbound on the sample complexity of distribution-free parity learning in the shuffle model.
Our proof has two main components. We first show how to construct a pan-private parity learner in the uniform distribution setting given a robust shuffle model distribution-free parity learner. Second, we show how to transform such a pan-private learner into a pan-private protocol for a  distinguishing task requiring an exponential number of samples. We get:

\medskip

\noindent{\bf Theorem~\ref{thm:LowerBoundParity}} (informal). {\em  For every distribution-free parity learning algorithm in the shuffle model the sample complexity is $n=\Omega(2^{d/2})$.}

\medskip

This result is complemented by a robust shuffle model protocol for distribution free parity parity with sample complexity $O(d2^{d/2})$.

\paragraph{Other related work.}
Also relevant to our work are the results of Chen, Ghazi, Kumar, and Manurangsi~\cite{ChenG0M21}. They  prove that the sample complexity of parity learning in the shuffle model is $\Omega(2^{d/(k+1)})$. Comparing with our results, their lowerbound depends on the message complexity of the protocol, whereas our bound holds regardless of the message complexity. On the other hand, our lowerbound holds for robust shuffle model protocols, whereas the result of Chen et al.\ does not require robustness.

\section{Preliminaries}

\subsection{Differential privacy, pan-privacy, and the shuffle model}

Let $X$ be a data domain. We say that two datasets $x,x'\in X^n$ are {\em neighboring} if they differ on exactly one entry, i.e., $|\{i:x_i\not=x'_i\}=1|$.
\begin{definition}[differential privacy~\cite{DMNS06}] A randomized mechanism $M:X^n\rightarrow Y$ preserves $(\varepsilon,\delta)$-differential privacy if for all neighboring $x,x'\in X^n$, and for all events $T\subseteq Y$,
$$\Pr[M(x)\in T]\leq e^\varepsilon \cdot \Pr[M(x')\in T]+\delta,$$
where the probability is over the randomness of the mechanism $M$.
\end{definition}

\begin{definition}[pan-privacy~\cite{DworkNPRY10}]
For an online mechanism $M:X^n\rightarrow Y$, let $S_{\leq t}(x)$ represent the internal state of $M(x)$ after receiving the $t$ first inputs $x_1,\ldots,x_t$. We say $M$ is $(\varepsilon, \delta)$-pan-private if for every two neighbouring datasets $x,x'\in X^n$, for every event $T\subseteq Y$, and for every $1\leq t\leq n$,
$$
\Pr[(S_{\leq t}(x), M(x)) \in T] \leq e^\varepsilon \Pr[(S_{\leq t}(x'), M(x')) \in T] +\delta,
$$
where the probability is over the randomness of the online mechanism $M$.
\end{definition}

A one round shuffle model mechanism $M:X^n\rightarrow Y$, as introduced in~\cite{CheuSUZZ19}, consists of three types of algorithms: (i) local randomizers $R_1,\ldots,R_n$ each maps an input $x_i\in X$ to a collection of messages from an arbitrary message domain; (ii) A shuffle $S$ receives a collection of messages and outputs them in a random order; and (iii) an analyzer algorithm $A$ maps a collection of messages  random permutation to an outcome in $Y$. Malicious users may avoid sending their messages to the shuffle. We denote such users by $\bot$.
The output of $M=((R_1,\ldots,R_n),S,A)$ is hence $A(S(\hat R_1(x_1),\ldots,\hat R_n(x_n)))$ where $\hat R_i = R_i$ for honest users and $\hat R_i = \bot$ for malicious users.
\begin{definition}[robust one-round shuffle model~\cite{BalcerCJM20}]
A one round shuffle model mechanism $M=((R_1,\ldots,R_n),S,A)$ is $\gamma$-robust and $(\varepsilon,\delta)$-differentially private if when at least $\gamma n$ of the parties are honest for all neighboring $x,x'\in X^n$ and for all events $T\subseteq Y$, 
$$\Pr[S(\hat R_1(x_1),\ldots,\hat R_n(x_n))\in T] \leq e^\varepsilon\cdot \Pr[S(\hat R_1(x'_1),\ldots,\hat R_n(x'_n))\in T] +\delta,$$
where the probability is over the randomness of $(\hat R_1,\ldots,\hat R_n)$ and the shuffle $S$.
\end{definition}

\remove{
\begin{theorem}[Advanced composition]
For every $\varepsilon,\delta,\delta'\geq0$, if mechanism $M_i$ for $i\in[k]$ is $(\varepsilon, \delta)$-differentially private, then $M'=(M_1,M_2,\ldots,M_k)$ is $(\varepsilon',k\delta+\delta')$-differentially private, where
$$
\varepsilon'=\sqrt{2k\ln(1/\delta')}\cdot\varepsilon+k\varepsilon(e^\varepsilon-1).
$$
Specifically, to ensure $(\varepsilon',k\delta+\delta')$-differential privacy, it suffices to set 
$$
\varepsilon=\frac{\varepsilon'}{2\sqrt{2k\ln(1/\delta')}}.
$$
\end{theorem}
}

\subsection{Private learning}

A concept class $C$ is a collection of predicates over the data domain $c_r:X \rightarrow\{\pm 1\}$. 
Let $P\in\Delta(X)$ be a probability distribution over the data domain $X$ and let $h:X\rightarrow\{\pm 1\}$. The generalization error of hypothesis $h$ with respect to the concept $c$ is  $\mbox{error}_P(c,h)=\Pr_{x\sim P}[h(x)\not=c(x)]$.
\begin{definition}[PAC learning~\cite{Valiant84}]\label{def:PAC}
A concept class $C$ is $(\alpha,\beta,m)$ PAC learnable if there exists an algorithm $L$ such that for all distributions  $P\in\Delta(X)$ and all concepts $c\in C$,
$$\Pr\left[\{x_i\}_{i=1}^m\sim P ; h\leftarrow L\Big(\{(x_i,c(x_i)\}_{i=1}^m\Big); \mbox{error}_P(c,h)\leq \alpha\right]\geq 1-\beta,$$
where the probability is over the choice of $x_1,\ldots,x_m$ i.i.d.\ from $P$ and the randomness of $L$.
\end{definition}

Note that Definition~\ref{def:PAC} is of an {\em improper} learner as the hypothesis $h$ need not come from the concept class $C$.

\begin{definition}[weight $k$ parity]
Let $\mbox{PARITY}_{d,k} = \{c_{r,b}\}_{r\subseteq[d], |r|\leq k, b\in\{\pm 1\}}$ where $c_{r,b}:\{\pm 1\}^d\rightarrow \{\pm 1\}$ is defined as $c_{r,b}(x)=b\cdot \prod_{i\in r} x_i$. Where $k=d$ we omit $k$ and write $\mbox{PARITY}_d$.
\end{definition}

\begin{definition}
A {\em distribution-free} parity learner is a PAC learning algorithm for $\mbox{PARITY}_{d,k}$.
A {\em uniform distribution} parity learner is a PAC learning algorithm for $\mbox{PARITY}_{d,k}$ where the underlying distribution $P$ is known to be uniform over $X=\{\pm 1\}^d$.
\end{definition}

\begin{definition}[private learning~\cite{KLNRS08}]
A concept class $C$ is private PAC learnable by algorithm $L$ with parameters $\alpha, \beta, m ,\varepsilon, \delta$, if $L$ is $(\varepsilon,\delta)$-differential private and  $L$  $(\alpha, \beta, m)$-PAC learns concept class $C$.
\end{definition}

\subsection{Hard tasks for pan-private mechanisms}\label{HardDist}

Cheu and Ullman~$\cite{CheuUllman20}$ provide a family of distributions $\{P_v\}$ for which the sample complexity of any pan-private mechanism distinguishing a randomly chosen distribution in $\{P_v\}$ from uniform is high. 
Let $X = \{\pm 1\}^d$ be the data domain. Let $0\leq \alpha \leq 1/2$, a non-empty set $\ell \subseteq [d]$, and a bit $b \in \{\pm 1\}$, define the distribution $P_{d,\ell,b,\alpha}$ to be
$$
P_{d,\ell,b,\alpha}(x) =
\left\{ \begin{array}{ll}
(1+2\alpha)2^{-d} & \mbox{if}~\prod_{i \in \ell}x_i = b\\
(1-2\alpha)2^{-d} & \mbox{if}~\prod_{i \in \ell}x_i = -b
\end{array}\right.
$$
Equivalently, $P_{d,\ell,b,\alpha}(x)=(1+2b\alpha\prod_{i \in \ell}x_i)\cdot 2^{-d}$.
Define the family of distributions
$$
\mathcal{P}_{d,k,\alpha} = \{P_{d,\ell,b,\alpha}(x): \ell \subseteq [d], |\ell| \leq k, b \in \{\pm 1\}\}.
$$

Let $M:X^n\rightarrow Y$ be a $(\varepsilon,\delta)$-pan-private. 
Let $P_{d,L,B,\alpha}$ be a distribution which is chosen uniformly at random from the family of distributions $\mathcal{P}_{d,k,\alpha}$, i.e., $L$ is a uniformly random subset of $[d]$ with cardinality $\leq k$ and $B\in_R\{\pm 1\}$. 

\begin{theorem}[\cite{CheuUllman20}, restated]\label{theorem:lowerbound}
Let $M$ be a $(\varepsilon,\delta)$-pan-private algorithm. If $d_{TV}(M(P_{d,L,B,\alpha}^n),M(U^n)) = T$ then 
$$n=\Omega\left(T\bigg/ \sqrt{\frac{\varepsilon^2 \alpha^2}{\binom{d}{\leq k}} + \delta \log\frac{\binom{d}{\leq k}}{\delta}} \right).$$
In particular, when $\delta \log\left(\binom{d}{\leq k}/\delta\right) =o\left(\varepsilon^2 \alpha^2/\binom{d}{\leq k}\right)$ we get that
$$n=\Omega\left(\frac{ T\cdot \sqrt{\binom{d}{\leq k}}}{\varepsilon\alpha}\right).$$


\end{theorem}


\subsection{Tail inequalities}
\begin{theorem}[Chebyshev's inequality]
Let X be a random variable with expected value $\mu$ and non-zero variance $\sigma^2$. Then  for any positive number $a$,
$${\Pr(|X-\mu |\geq a )\leq {\frac {\sigma^2}{a^{2}}}.}$$
\end{theorem}

\remove{
\begin{theorem}[Chernoff Bound]
Let $X$ be a random variable and $\mathbb{E}(X)=\mu$, for all $\delta>0$,
$$
\Pr[X>(1+\delta)\mu]\leq e^{\frac{-\delta^2\mu}{2+\delta}},
$$
$$
\Pr[X<(1-\delta)\mu]\leq e^{\frac{-\delta^2\mu}{2+\delta}}.
$$
\end{theorem}
}

\section{A lowerbound on the sample complexity of parity learning in the shuffle model}

\subsection{From robust shuffle model parity learner to a pan-private parity learner}

We show how to construct, given a robust shuffle model distribution-free parity learner, a uniform distribution pan-private parity learner.
Our reduction--Algorithm LearnParUnif--is described in~Algorithm~\ref{alg:LearnParUnif}. We use a similar technique to the padding presented in~\cite{BalcerCJM20, CheuUllman20}, with small modifications. To allow the shuffle model protocol use different randomzers $R_1,\ldots,R_n$, the pan-private learner applies these randomizers in a random order (the random permutation $\pi$). The padding is done with samples of the form $(0^d,\hat b)$ where $\hat b$ is a uniformly random selected bit. Finally, as in~\cite{CheuUllman20} the number of labeled samples which the pan-private algorithm considers from its input is binomially distributed, so that if $(x_i,y_i)$ are such that $x_i$ is uniform in $X$ and $y_i = c_{r,b}(x_i) = b\cdot \prod_{i\in r}x_i$ then (after a random shuffle) the input distribution presented to the shuffle model protocol is statistically close to a mixture of the two following distributions: (i) a distribution where $\Pr[(x_i,y_i) =(0^d, \hat b)]=1$ and (ii) a distribution where $x_i$ is uniformly selected in $\{\pm 1\}^d$ and $y_i=c_{r,b}(x_i)$.

\begin{algorithm}[ht]
\SetAlgoLined
\LinesNumbered
Let $M=((R_1,\ldots,R_n),S,A)$ be a 1/3-robust differentially private distribution parity learner. 

\KwIn{$n/3$ labeled examples $(x_i,y_i)$ where $x_i\in X$ and $y_i\in \{\pm 1\}$.}

Randomly choose a permutation $\pi: [n]\rightarrow [n]$.

Randomly choose $\hat b\in_R\{\pm 1\}$.

Create initial state $s_0 \leftarrow S(R_{\pi(1)}(0^d,\hat b),\ldots, R_{\pi(n/3)}(0^d,\hat b))$.

Sample $N' \sim \textbf{Bin}(n,2/9)$.

Set $N' \leftarrow min(N', n/3)$.

\For{$i \in [n/3]$}{
 \eIf{$i\in[N']$}{
    $w_i \leftarrow (x_i,y_i)$}{
    $w_i \leftarrow (0^d, \hat b)$}
 $s_i \leftarrow S(s_{i-1}, R_{\pi(n/3+i)}(w_i))$
 }
 $s_{final} \leftarrow S(s_{n/3}, R_{\pi(2n/3+1)}(0^d,\hat b),\ldots, R_{\pi(n)}(0^d,\hat b))$
 
 \Return{$A(s_{final})$}
 \caption{LearnParUnif, a uniform distribution pan-private parity learner\label{alg:LearnParUnif}}
\end{algorithm}

\begin{proposition}
Algorithm LearnParUnif is $(\varepsilon,\delta)$-pan-private.
\end{proposition}

\begin{proof}[Proof sketch, following~\cite{BalcerCJM20,CheuUllman20}]
Let $x$ and $x'$ be two neighboring data sets and let $j$ be the index where $x$ and $x'$ differ. Let $1 \leq t \leq n/3$ be the time an adversary probes into the algorithm's memory. 

If $t\geq j$ then 
$S_{\leq t}=(S\circ (R^{\pi(1)},\ldots, R^{\pi(n/3+t)}))((0^d,b)^{n/3},w_{1},\ldots,w_{t})$ and, as $M$ is a robust differentially private mechanism $S_{\leq t}$ preserves $(\varepsilon,\delta)$-differential privacy. Because $A(s_{final})$ is post-processing of $S_{\leq t}$ the outcome of $LearnParUnif$ is $(\varepsilon,\delta)$-pan-private.

If $t<j$ then $S_{\leq t}(x)$ is identically distributed to $S_{\leq t}(x')$. Note that as $M$ is a robust differentially private mechanism we get that
$$\sigma = (S\circ (R^{\pi(n/3+t+1)},\ldots,R^{\pi(n)}))(w_{t+1},\ldots,w_{N'},(0^d,b),\ldots,(0^d,b))$$ preserves $(\varepsilon,\delta)$-differential privacy. To conclude the proof, note that $(S_{\leq t}(x), A(s_{final}))$ is the result of post-processing $\sigma$.
\end{proof}

\begin{proposition}[learning]
Let $M$ be a $(\alpha,\beta,m)$ distribution free parity learner, where $\alpha,\beta <1/4$ and $m=n/9$. Algorithm LearnParUnif is a uniform distribution parity learner that with probability at least $1/4$ correctly identifies the concept $c_{r,b}$.
\end{proposition}

\begin{proof}[Proof sketch]
Algorithm LearnParUnif correctly guesses the label $b$ for $0^d$ with probability $1/2$. Assuming $\hat b = b$ the application of $M$ uniquely identifies $r,b$ with probability at least $1/2$. Thus, $LearnParUnif$ recovers $c_{r,b}$ with probability at least 1/4.
\end{proof}

\subsection{From pan-private parity learner to distinguishing hard distributions}
In this section, we use Theorem~\ref{theorem:lowerbound} to obtain a lowerbound on the sample complexity of parity learning in the shuffle model. 
In Algorithm~\ref{DistParity}, we provide a reduction from identifies the hard distribution $P_{d,\ell,b,1/2}$ presented in section~\ref{HardDist} to pan-private parity learning. 

\begin{algorithm}\label{DistParity}
\caption{IdentifyHard, a pan-private for identifying the distribution $P_{d,\ell,b,1/2}$ }
\SetAlgoLined
\LinesNumbered
Let $\Pi$ be a pan-private uniform distribution parity learner.

\KwIn{A sample of $n$ examples $z=(z_1, z_2, \ldots, z_n)$, where each example is of the form $z_j=(z_j[1],z_j[2],\ldots,z_j[d])\in\{\pm1\}^d$}

Randomly choose $i^* \in_R [d]$.

/* Apply the uniform distribution parity learner $\Pi$: */

\For{$j\in[n]$}{
$y_j \leftarrow z_j[i^*]$

$x_j = z_j$ 

$x_j[i^*]=\bot$ /* i.e., $x_j$ equals $z_j$ with entry $i^*$ erased */

Provide $(x_j,y_j)$ to $\Pi$. \label{step:ApplyLearner}
}

$(r,b) \leftarrow \Pi((x_1, y_1), \ldots, (x_n,y_n))$

$\ell \leftarrow r \cup \{i^*\}$

\Return{$(\ell,b)$}
\end{algorithm}

\begin{observation}
The pan-privacy of Algorithm~\ref{DistParity} follows from the pan-privacy of algorithm $\Pi$.
\end{observation}

\begin{proposition}
Given a uniform distribution parity learner that with probability at least $1/4$ correctly identifies the concept $c_{r,b}$, algorithm \ref{DistParity} can correctly identify the distribution $P_{d,\ell,b,1/2}$ with probability at least $\frac{|\ell|}{4d}$.
\end{proposition}

\begin{proof}
Note that with probability $|\ell|/d$ we get that $i^*\in \ell$, in which case the inputs $x_1,\ldots,x_n$ provided to the learner $\Pi$ in Step~\ref{step:ApplyLearner} are uniformly distributed in $\{\pm 1\}^{d-1}$ and $y_j = b\cdot \prod_{i\in \ell \setminus \{i^*\}} x_j[i]$, i.e., the inputs to $\Pi$ are consistent with the concept $c_{\ell \setminus \{i^*\}, b}$. 
\end{proof}

On the uniform distribution, the generalization error of any parity function is $1/2$. On $P_{d,\ell,b,1/2}$ Algorithm~\ref{DistParity} succeeds with probability $|\ell|/4d$ to identify $\ell, b$. Algorithm \ref{correctness} evaluates the generalization error of the concept learned in algorithm~\ref{DistParity} towards exhibiting a large total variance distance on $P_{d,L,B,1/2}^n$ and $U^n$.

\begin{algorithm}\label{correctness}
\DontPrintSemicolon
\caption{$DistPU$: Distinguisher for $P_{d,L,B,1/2}^{n+m}$ and $U^{n+m}$}
\SetAlgoLined
\LinesNumbered
Let $M=((R_1,\ldots,R_n),S,A)$ be the pan-private algorithm described in Algorithm~\ref{DistParity}.

\KwIn{A sample of $m+n$ examples $z=(z_1, z_2, \ldots, z_{n+m})$, where $m=\max\{512d/k,64\sqrt{2d/k}/\varepsilon\}$ and each example is of the form $z_j=(z_j[1],z_j[2],\ldots,z_j[d])\in\{\pm1\}^d$.}

Let $(\ell,b)$ be the outcome of executing $M$ On the first $n$ examples $z_1,\ldots,z_n$.

$c \leftarrow \textbf{Lap}(1/\varepsilon)$

\For{$i\in[m]$}{
    \lIf{$\prod_{j \in \ell}z_{i+n}[j] = b$}{
        $c\leftarrow c+1$
    }
}

$c^* \leftarrow c+ \textbf{Lap}(1/\varepsilon)$

\leIf{$c^* \geq 3m/4$}{
    \Return{1}
}{
    \Return{0}
}
\end{algorithm}

Observe that if $z\sim P_{d,L,B,1/2}^{n+m}$ then in every execution of Algorithm~\ref{correctness} there exists $\ell \subset [d]$ of cardinality at most $k$ and $b\in\{\pm 1\}$ such that $z\sim P_{d,\ell,b,1/2}^{n+m}$.

\begin{proposition}
$\Pr_{z \sim P_{d,\ell,b,1/2}^{n+m}}[DistPU(z)=1]\geq\frac{|\ell|}{8d}$.
\end{proposition}
\begin{proof}
For any $z \sim P_{d,\ell,b,1/2}^{n+m}$, we always have $\prod_{i \in \ell}z_i = b$, so
$$
\begin{aligned}
\Pr[DistPU(z)=1]&\geq
\Pr[DistPU~\mbox{correctly identifies}~(\ell,b)] \cdot \Pr[c^* \geq 3m/4] \\
&\geq \frac{|\ell|}{4d}\cdot\Pr[\textbf{Lap}(1/\varepsilon)+\textbf{Lap}(1/\varepsilon)\geq -m/4] \\
&\geq\frac{|\ell|}{4d}\cdot\frac{1}{2} \quad\quad\quad\quad  \mbox{(By symmetry of \textbf{Lap} around 0)}\\
&=\frac{|\ell|}{8d}.
\end{aligned}
$$
\end{proof}

\begin{proposition}
$\Pr_{z \sim U^{n+m}}[DistPU(z)=1]\leq\frac{k}{64d}$.
\end{proposition}
\begin{proof}
For all $(\ell, b)$, we have that 
$\Pr_{z\sim U}[\prod_{j\in\ell}z[j]=b]=1/2$, so we have
$$
\begin{aligned}
\Pr_{z \sim U^{n+m}}[DistPU(z)=1]&=\Pr[\textbf{Bin}(m,1/2)+\textbf{Lap}(1/\varepsilon)+\textbf{Lap}(1/\varepsilon)\geq 3m/4]\\
&\leq \Pr[\left|\textbf{Bin}(m,1/2)+\textbf{Lap}(1/\varepsilon)+\textbf{Lap}(1/\varepsilon)-m/2\right|\geq m/4]\\
&\leq \frac{m/4+2/\varepsilon^2+2/\varepsilon^2}{m^2/16}\quad\quad \mbox{(Chebyshev's inequality)}\\
&=4/m+64/\varepsilon^2m^2\\
&\leq\frac{k}{128d}+\frac{k}{128d} = \frac{k}{64d}.
\end{aligned}
$$
\end{proof}


\begin{proposition}
$d_{TV}(DistPU(U^{n+m}),DistPU(P_{d,L,B,1/2}^{n+m}))\geq\frac{k}{64d}$. 
\end{proposition}
\begin{proof}
\begin{eqnarray*}
\lefteqn{d_{TV}(DistPU(U^{n+m}),DistPU(P_{d,L,B,1/2}^{n+m}))} \\
&\geq & \Pr_{z \sim P_{d,L,B,1/2}^{n+m}}[DistPU(z)=1]-\Pr_{z \sim U^{n+m}}[DistPU(z)=1]\\
&= & \sum_{\ell\in[d],|\ell|\leq k,b\in\{\pm1\}}\Pr_{z \sim P_{d,\ell,b,1/2}^{n+m}}[DistPU(z)=1]\cdot\Pr[(L,B)=(\ell,b)]-\Pr_{z \sim U^{n+m}}[DistPU(z)=1]\\
&\geq & \sum_{\ell\in[d],k/2\leq|\ell|\leq k,b\in\{\pm1\}}\Pr_{z \sim P_{d,\ell,b,1/2}^{n+m}}[DistPU(z)=1]\cdot\Pr[(L,B)=(\ell,b)]-\Pr_{z \sim U^{n+m}}[DistPU(z)=1]\\
&\geq & \sum_{\ell\in[d],k/2\leq|\ell|\leq k,b\in\{\pm1\}} \frac{k}{16d}\cdot\Pr[(L,B)=(\ell,b)]-\Pr_{z \sim U^{n+m}}[DistPU(z)=1]\\
& = & \frac{k}{16d}  \cdot \Pr[|L|\geq k/2]-\Pr_{z \sim U^{n+m}}[DistPU(z)=1]\\
& = & \frac{k}{16d} \cdot \frac{\binom{d}{\leq k}- \binom{d}{\leq k/2}}{\binom{d}{\leq k}}-\Pr_{z \sim U^{n+m}}[DistPU(z)=1] \geq \frac{k}{32d}-\Pr_{z \sim U^{n+m}}[DistPU(z)=1] \geq \frac{k}{64d}.
\end{eqnarray*}
\end{proof}
The last inequality follows from $\frac{\binom{d}{\leq k}- \binom{d}{\leq k/2}} {\binom{d}{\leq k}} \geq 1/2$.\footnote{If $k=d$ then $\binom{d}{\leq k} \geq 2\binom{d}{\leq k/2}$. 
Otherwise ($k<d$) we get for $0\leq i \leq \lfloor k/2\rfloor$ that the difference between $\lfloor k/2\rfloor +1+i$ and $d/2$ is smaller than the difference between $\lfloor k/2\rfloor -i$ and $d/2$ hence $\binom{d}{\lfloor k/2\rfloor-i} < \binom{d}{\lfloor k/2\rfloor +1+i}$, thus $\binom{d}{\leq k} = \sum_{0\leq i\leq \lfloor k/2\rfloor}{\binom{d}{i}} + \sum_{\lfloor k/2\rfloor + 1\leq i\leq k}{\binom{d}{i}} > 2 \sum_{0\leq i\leq \lfloor k/2\rfloor}{\binom{d}{i}} = 2\binom{d}{\leq k/2}$.}

In particular, for all $k$ we get that $d_{TV}(DistPU(U^{n+m}),DistPU(P_{d,L,B,1/2}^{n+m})) \geq k/64d$ and for $k=d$ we get $d_{TV}(DistPU(U^{n+m}),DistPU(P_{d,L,B,1/2}^{n+m})) \geq 1/64$.

\begin{theorem}\label{thm:LowerBoundParity}
For any $(\varepsilon,\delta,1/3)$-robust private distribution-free parity learning algorithm in the shuffle model, where $\varepsilon=O(1)$, the sample complexity is
$$n=\Omega\left(\frac{2^{d/2}}{\varepsilon}\right).$$
\end{theorem}

\begin{proof}
Let $k=d$, applying Theorem \ref{theorem:lowerbound}, $DistPU$ has sample complexity 
$$
n+m =\Omega\left(\frac{2^{d/2}}{\varepsilon}\right).
$$
Since $k\geq 1$, $\varepsilon=O(1)$, $m=O(d/\varepsilon)$. By the of Algorithm $DistPU$ from a $(\varepsilon,\delta,1/3)$-robust private parity learning algorithm, any $(\varepsilon,\delta,1/3)$-robust private parity learning algorithm has sample complexity
$$
n=\Omega\left(\frac{2^{d/2}}{\varepsilon}\right).
$$
\end{proof}

\subsection{Tightness of the lowerbound}

We now observe that Theorem~\ref{thm:LowerBoundParity} is tight as there exists a $1/3$-robust agnostic parity learner in the shuffle model with an almost matching sample complexity. 
For every possible hypothesis $(\ell,b)$ (there are $2^{d+1}$ hypotheses) the learner estimates the number of samples which are consistent with the hypothesis, $c_{\ell,b}=|\{i: b\cdot \prod_{j\in\ell} x_i[j]=y_i\}|$.

One possibility for counting the number of consistent samples is to use the protocol by Balle et al.~\cite{BalleBGN19a} which is an $(\epsilon,\delta)$-differentially private one-round shuffle model protocol for estimating $\sum a_i$ where $a_i\in[0,1]$. 
The outcome of this protocol is statistically close to $\sum a_i + DLap(1/\epsilon)$ and the statistical distance $\delta$ can be made arbitrarily small by increasing the number of messages sent by each agent. (We use the notation $DLap(1/\epsilon)$ for the Discrete Laplace distribution, where the probability of selecting $i\in \mathbb{Z}$ is proportional to $e^{-\epsilon |i|}$). The protocol uses the divisibility of Discrete Laplace random, generating Discrete Laplace noise $\nu$ as the sum of differences of Polya random variables: $\nu = \sum_{i=1}^n Polya(1/n, e^{-\epsilon}) - Polya(1/n, e^{-\epsilon})$. 
To make the protocol $\gamma$-robust, we slightly change the noise generation to guarantee $(\epsilon,\delta)$ differential privacy in the case where only $n/3$ parties participate in the protocol. 
This can be done by changing the first parameter of the Polya random variables to $3/n$ resulting in $\nu = \sum_{i=1}^n Polya(3/n, e^{-\epsilon}) - Polya(3/n, e^{-\epsilon})$. Observe that $\nu$ is distributed as the sum of three independent $DLap(1/\epsilon)$ random variables. 
Using this protocol, it is possible for the analyzer to compute a noisy estimate of the number of samples consistent with each hypothesis, $\tilde c_{\ell,b}= c_{\ell,b}+\nu$, and then output $(\ell,b)=\argmax_{\ell,b}(\tilde c_{\ell,b})$.
The sample complexity of this learner is $O_{\alpha,\beta,\epsilon,\delta}(d2^{d/2})$. 

\bibliographystyle{plain}
\bibliography{shuffleparity}

\end{document}